\documentclass{article}

\usepackage{microtype}
\usepackage{physics}
\usepackage{graphicx}
\usepackage{subfigure}
\usepackage{footmisc}
\usepackage{booktabs} 
\usepackage{multirow}
\usepackage{standalone}
\usepackage{tikz}
\usepackage{pgfplots}
\pgfplotsset{compat=1.14}   
    \usepgfplotslibrary{colorbrewer}
    \pgfplotsset{
        cycle list/Dark2,
        cycle multiindex* list={
            mark list*\nextlist
            Dark2\nextlist
        },
    }

\usepackage{hyperref}

\usepackage{amsmath,amssymb,amsthm}

\newcommand{\R}{\mathbb{R}}

\newcommand{\E}{\mathbb{E}}
\newcommand{\T}{\top}
\def\*#1{\mathbf{#1}}

\renewenvironment{proof}{\begin{newproof}}{\end{newproof}\qed}
\newenvironment{psketch}{%
  \proof}{\endproof}

\def\<#1\>{\langle#1 \rangle}
\newtheorem{theorem}{Theorem}[section]

\newtheorem{lemma}[theorem]{Lemma}

\DeclareMathOperator*{\argmax}{argmax} 

\usepackage[accepted]{icml2020}

\icmltitlerunning{Generalized Multi-view Shared Subspace Learning using View Bootstrapping}

\begin{document}

\twocolumn[
\icmltitle{Generalized Multi-view Shared Subspace Learning using View Bootstrapping}



\icmlsetsymbol{equal}{*}

\begin{icmlauthorlist}
\icmlauthor{Krishna Somandepalli}{to}
\icmlauthor{Shrikanth Narayanan}{to}
\end{icmlauthorlist}

\icmlaffiliation{to}{Department of Electrical Engineering, University of Southern California, USA}

\icmlcorrespondingauthor{Krishna Somandepalli}{somandep@usc.edu}

\icmlkeywords{Machine Learning, ICML}

\vskip 0.3in
]



\printAffiliationsAndNotice{}  

\begin{abstract}
A key objective in multi-view learning is to model the information common to multiple parallel views of a class of objects/events to improve downstream learning tasks. In this context, two open research questions remain: How can we model hundreds of views per event? Can we learn robust multi-view embeddings without any knowledge of how these views are acquired? We present a neural method based on multi-view correlation to capture the information shared across a large number of views by subsampling them in a view-agnostic manner during training. To provide an upper bound on the number of views to subsample for a given embedding dimension, we analyze the error of the bootstrapped multi-view correlation objective using matrix concentration theory. Our experiments on spoken word recognition, 3D object classification and pose-invariant face recognition demonstrate the robustness of view bootstrapping to model a large number of views. Results underscore the applicability of our method for a view-agnostic learning setting.
\end{abstract}

\vspace{-0.2in}

\section{Introduction}
\label{intro}
Across many application domains, we often rely on data collected from multiple views of a target object/event to learn a reliable and comprehensive representation. 
This group of (machine) learning problems is referred to as multi-view learning. 
A distinguishing feature of this paradigm is that the different views of a given instance share an association or a~\textit{correspondence} that can be exploited to build more informed models of the observed event \cite{xu2013survey}.
Much like the process by which humans learn by reconciling different views of information that may appear conflicting \cite{klemen2012current}, data from different views contain both contrasting and complementary knowledge that can be used to offer robust learning solutions.


We define a \textit{view} as data that is sampled from observing an object/event at different states or with different instruments to capture its various presentations. For example, a person's face photographed at different angles or audio, language and visuals in an movie. 
The objective of multi-view learning is to learn vector representations (embeddings/features) that are discriminative of the underlying events by explicitly factoring in/out the shared correspondence between the many views.
These embeddings can provide robust features for downstream tasks such as classification and clustering, e.g., text-to-image retrieval \cite{dorfer2018end} and bilingual word embeddings \cite{wang2015deep}. They can also be used in an unsupervised fashion to uncover the inherent structure in such data, e.g., learning common components from brain signals across individuals \cite{parra2018correlated}.

Multi-view learning solutions have explored various ways to model the \textit{correspondence} between multiple views to fuse the knowledge across them. They can be broadly categorized into (1) subspace alignment methods, (2) generative models and (3) fusion-based methods \cite{li2018survey}. The present work can be classified as subspace-alignment, which deals with learning projections between two or more views to maximize the similarity.
Most existing subspace-alignment methods learn multi-view representations by estimating at least one distinct projection matrix per view, often assuming that the view information for the probing sample is available at training/testing. 
Considering the sheer scale of multi-view problems--amount of data and number of views--two critical questions arise: how can we model hundreds of views of an event, and can we learn the multi-view representations effectively in a view-agnostic fashion?

In this paper, we build upon the work by Somandepalli et. al., \yrcite{somandepalli2019multiview,somandepalli2019arxiv} where a \textit{multi-view correlation} objective (mv-corr) was proposed to learn shared subspaces across multiple views. Data from different views is transformed using identical neural networks (\textit{sub-networks}) to obtain view-invariant embeddings discriminative of the underlying event.
We advance this framework along two directions: First, we explore view bootstrapping during training to be able to incorporate a large number of views. We provide a theoretical analysis for the bootstrapped mv-corr objective and derive an upper bound for the number of views to subsample with respect to the embedding dimension. This result is significant because it allows us to determine the number of sub-networks to use in the mv-corr framework.

Second, we conduct several experiments to benchmark the performance of view-bootstrapping for downstream learning tasks and highlight its applicability for modeling a large number of views in a view-agnostic fashion.
In practice, this framework only needs to know that the sample of views considered at each training iteration have a \textit{correspondence}. 
That is, the multiple views are obtained from observing the same underlying event. 
A natural example for this setting is audio recordings from multiple microphones distributed in a conference room. In this example, we can use the timestamps to construct a correspondence. 
This method can also be used for applications such as pose-invariant face recognition in a semi-supervised setting. We do not need the pose information (view-agnostic) or the total number of classes during training. All we need to know is that the different face images are of the same person. 

\section{Related Work}
\label{background}
\subsection{Subspace alignment for more than two views}
Widely used correlation-based methods include canonical correlation analysis (CCA) \cite{hotelling1992relations} and its deep learning versions \cite{andrew2013deep, dumpala2018sentiment} that can learn non-linear transformations of the two views to derive maximally correlated subspaces. Several metric-learning based methods were proposed to extend CCA for multiple views by learning a view-specific or view-invariant feature space by transforming data. For example, generalized CCA \cite{horst1961generalized,benton2017deep} and multi-view CCA \cite{chaudhuri2009multi}. Their applications include audio-visual speaker clustering and phoneme classification from speech and articulatory information.

In a supervised setting, a discriminative multi-view subspace can be obtained by treating labels as an additional view. 
Prominent examples of this idea include generalized multi-view analysis (GMA, \citealt{sharma2012generalized}), partial least squares regression based methods \cite{cai2013regularized} and multi-view discriminant analysis (MvDA, \citet{kan2015multi}). They were effectively used for applications such as image captioning and pose-invariant face recognition. However the generalizability of these methods to hundreds of views remains to be explored. 
\subsection{View-agnostic multi-view learning}
The subspace methods discussed thus far assume that the view information is readily available during training and testing. 
For instance, GMA and MvDA estimate a within-class scatter matrix specific to each view. 
In practice, view information may not be available for the probe data (e.g., pose of a face during testing). 
A promising direction to address this problem was proposed by Ding and Fu \yrcite{ding2014low,ding2017robust}. To eliminate the need for view information of the probe sample, a low-rank subspace representation was used to bridge the view-specific and view-invariant representations. Here, a single projection matrix per view was used which would scale linearly with increasing number of views.
\subsection{Domain adaptation in a multi-view paradigm}
A recent survey by \citet{ding2018robust} presents a unified learning framework mapping out the similarities between multi-view learning and domain adaptation. 
Typical domain adaptation methods seek domain-invariant representations which are akin to view-invariant representations if we treat different domains as views. 
The benefit of the multi-view paradigm in this context is that the variabilities associated with multiple views can be \textit{washed out} to obtain discriminative representations of the underlying classes. 

This formulation is particularly useful in the domain of speech/audio processing for applications such as wake-word recognition \cite{kepuska2009novel}. Here we need to recognize a keyword (e.g., ``Alexa'', ``OK Google'', ``Siri'') no matter who says it or where it is said (i.e., the specific background acoustic conditions). Speaker verification methods based on joint factor analysis \cite{dehak2009support} and total variability modeling \cite{dehak2011front} have explored the ideas of factoring out the speaker-dependent factors and speaker-independent factors to obtain robust speaker representations in the context of domain adaption. 
Recently, \citet{somandepalli2019multiview} showed that a more robust speech representation can be obtained by explicitly modeling multiple utterances of a word as corresponding views. 

\subsection{Views vs. Modalities}\label{sec:spiel}

Following ideas proposed in the review by \citet{ding2018robust}, we delineate two kinds of allied but distinct learning problems: \textit{multi-view} and \textit{multi-modal}. 
In related work of this domain (See surveys by \citealt{zhao207multi,ding2018robust,li2018survey}), the two terms are used interchangeably. We however distinguish the two concepts to facilitate modeling and analysis. 
Multiple views of an event can be modeled as samples drawn from identically distributed random processes, e.g., a person's face at different poses. 
However, the individual modalities in multi-modal data need not arise from identically distributed processes, e.g., person's identity from their voice, speech and pose. 

In this work, we focus on multi-view problems, specifically to learn embeddings that capture the shared information across the views. The premise that multiple views can be modeled as samples from identically distributed processes not only facilitates the theoretical analysis of the mv-corr objective, but also helps us to formulate domain adaptation problems in a multi-view paradigm; particularly, for applications that need to scale for hundreds of views (e.g., speaker-invariant word recognition).
While it should be noted that the methods explored in this work may not be directly applied to multi-modal problems where we are generally interested to capture both modality-specific and modality-shared representations, the theory developed in this work can be extended to other methods such as GMA \cite{sharma2012generalized} and multi-view deep network \cite{kan2016multi} for the broader class of multi-modal problems.

\begin{figure*}[t!]
\vskip 0.2in
\begin{center}
\scalebox{0.9}{\centerline{\includegraphics[width=2\columnwidth]{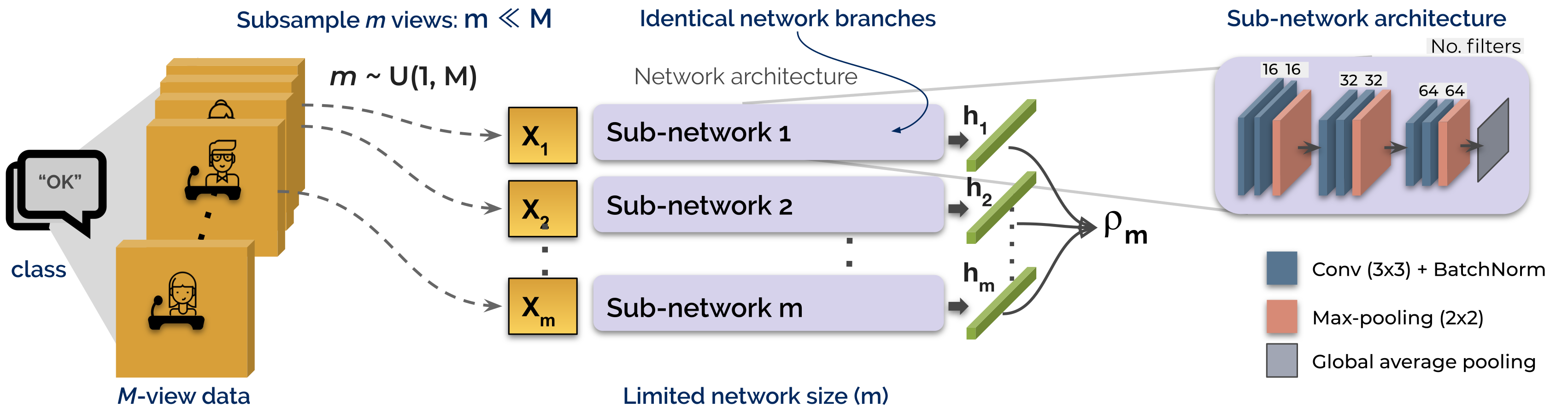}}}
\caption{Schematic of view bootstrapping for multi-view shared subspace learning. Inset: example sub-network architecture}
\label{fig:schematic}
\end{center}
\vskip -0.2in
\end{figure*}





\section{Proposed Approach}
\label{methods}
We first review the multi-view correlation (mv-corr) objective developed by Somandepalli et. al.,~\yrcite{somandepalli2019multiview,somandepalli2019arxiv}. Next, we consider practical aspects for using this objective in a deep learning framework followed by view-bootstrapping. 
Then, we develop a theoretical analysis to understand the error of the bootstrapped mv-corr objective.

\subsection{Multi-view correlation (mv-corr)}
Consider $N$ samples of $d$-dimensional features sampled by observing an object/event from $M$ different views. 
Let $\*X_l \in \mathbb{R}^{d\times N}: l=1,..., M$, be the data matrix for the $l$-th view with columns as mean-zero features. 
We can use the same feature dimension $d$ across all views because we assume that that the multiple views are sampled from identical distributions (See Sec.~\ref{sec:spiel}). 
We describe the mv-corr objective in the context of CCA. 
The premise of applying CCA-like approaches to multi-view learning is that the \textit{inherent variability associated with a semantic class is uncorrelated across multiple views} to represent the signal shared across the views.
For $M=2$, CCA finds projections of same dimensions $\*v_1\text{ and }\*v_2$ in the direction that maximizes the correlation between them. Formally,
\begin{eqnarray}\label{eqn:cca}
    (\*v_1^*, \*v_2^*)
    = \argmax_{\*v_1,\*v_2 \in \R^{d}}\frac{\*v_1^{\T}\*{\Sigma}_{12}\*v_2}{\sqrt{\*v_1^{\T}\*\Sigma_{11}\*v_1\*v_2^{\T}\*\Sigma_{22}\*v_2}}
\end{eqnarray}
where $\*\Sigma_{12}$ is the cross-covariance and $\*\Sigma_{11}\,,\*\Sigma_{22}$ are the covariance terms for the two views. 
To extend the CCA formulation for more than two views, we consider the sum of all pairwise covariance terms. 
That is, find a projection matrix or a \textit{multi-view shared subspace} $\*W\in\R^{k\times d}$ that maximizes the ratio of the sum of between-view over within-view covariances in the projected space:
\begin{eqnarray}\label{eqn:mcca}
    \*W^*
    = \argmax_{\*W}\frac{\*W^{\T}\big(\*X_1\*X_2^{\T} +\dotsc+\*X_{M-1}\*X_{M}^{\T}\big)\*W}{\*W^{\T}\big(\*X_1\*X_1^{\T} +\dotsc+\*X_{M}\*X_{M}^{\T}\big)\*W}
\end{eqnarray}
We refer to the numerator and denominator covariance sums in Eq.~\ref{eqn:mcca} as \textit{between-view covariance} $\*R_b$ and \textit{within-view covariance} $\*R_w$ which are sums of $M(M-1)$ and $M$ covariance terms, respectively. Because we assume feature columns in $\*X_l$ to be mean-zero, we estimate the covariance matrices as a cross product without loss of generality.

We now define a multi-view correlation $\*\Lambda$ as the normalized ratio of between- and within-view covariance matrix:
\begin{eqnarray}\label{eqn:mv-corr}
    \*\Lambda = \max_{\*W}\frac{1}{M-1}\frac{\*W^{\T}\*R_b\*W}{\*W^{\T}\*R_w\*W}
\end{eqnarray}
here, the common scaling factor $M(N-1)$ in the covariance estimates are omitted from the ratio. 

A version of this \textit{ratio of covariances} has been considered in several related multi-view learning methods. One of the earliest works by \citet{hotelling1992relations} presented a similar formulation for scalars, also referred to as multi-set CCA by some works (e.g., \citealt{parra2018correlated}). Notice that this ratio is similar to the use of between-class and within-class scatter matrices in linear discriminant analysis (LDA, \citealt{fisher1936use}) and more recently in multi-view methods such as GMA and MvDA. Another version of this ratio known as the intraclass correlation coefficient \cite{bartko1966intraclass} has been extensively used to quantify test-retest repeatability of clinical measures (e.g., \citealt{somandepalli2015short}).

The primary difference of mv-corr formulation from these methods is that it does not consider the class information explicitly while estimating the covariance matrices. All we need to know is that the subset of $M$ views \textit{correspond} to the same object/event. Additionally we consider the sum of covariances for all pairs of views, eliminating the need for view-specific transformation which enables us to learn the shared subspace $\*W$ in a view-agnostic manner. 
On the downside, we only capture the shared representation across multiple views and discard view-specific information which may be of interest for some multi-modal applications.

\subsection{Implementation and practical considerations}
Using ideas similar to the deep variants of CCA \cite{andrew2013deep} and LDA \cite{dorfer2015deep}, we can use deep neural networks (DNN) to learn non-linear transformations of the multi-view data to obtain (possibly) low-dimensional representations.
In Eq.~\ref{eqn:mv-corr}, the solution $\*W$ jointly diagonalizes the two covariances $\*R_b$ and $\*R_w$ because $\*W$ is their common eigenspace. Thus, we use the trace ($\Tr$) form of Eq.~\ref{eqn:mv-corr} to fashion a loss function, $\rho_M$ for batch optimization in DNN for data from $M$ views.
\begin{align}\label{eqn:rho-M}
    \rho_{M}
    = \max_{\*W}\frac{1}{d(M-1)}\frac{\Tr(\*W^{\T}\*R_b\*W)}{\Tr(\*W^{\T}\*R_w\*W)} 
\end{align}

The DNN framework for mv-corr consists of one network per view~$l$, referred to as $l^{\text{th}}$ \textit{sub-network} denoted by~$f_l$.
The architecture of the sub-network is the same for multiple views and the weights are not shared across the sub-networks for any layer.
The output from the top-most layer of each sub-network is passed to a fully-connected layer of $d$ neurons. 
Let $\*H_l=f_l(\*X_l)\in\R^{d\times N}$ be the activations from this last layer where $N$ is now the batch size. 
Thus, for each batch we estimate the between- and within-view covariances $\*R_b$ and $\*R_w$ using $\*H_l\,,l=1,\dotsc,M$ to compute the loss in Eq.~\ref{eqn:rho-M}. The subspace $\*W$ is obtained by solving the generalized eigenvalue (GEV) problem  using Cholesky decomposition.

\textbf{Total view covariance}: For a large number of views $M$, estimating $\*R_b$ in each batch is expensive as it is $\mathcal{O}(M^2)$. We instead compute a total-view covariance term $\*R_t$ which only involves estimating a single covariance for the sum of all views and is $\mathcal{O}(M)$, and then estimate $\*R_b=\*R_t-\*R_w$.
See Supplementary (Suppl.) methods S1 for the proof.
\begin{eqnarray}\label{eqn:rt}
    \*R_t = \*R_b + \*R_w = \frac{1}{M}\bigg(\sum_{l=1}^{M}\*X_l\bigg)\bigg(\sum_{l=1}^{M}\*X_l\bigg)^{\T}
\end{eqnarray}
\newline\textbf{Choosing batch size}: A sample size of $\mathcal{O}(d\log d)$ is sufficient to approximate the sample covariance matrix of a general distribution in $\R^d$ \cite{vershynin2010introduction}. Thus we choose a batch size of $N=\text{ceil}(d\log d)$ for a network with $d$-dimensional embeddings. In our experiments, choosing $N<d\log d$ was detrimental to model convergence.

\textbf{Regularize $\*R_w$}: Maximizing $\rho_M$ (Eq.~\ref{eqn:rho-M}) corresponds to maximizing the mean of eigenvalues of $\*R_w^{-1}\*R_b$. Estimating $\*R_w$ with rank deficient $\*H_l$ may lead to spuriously high $\rho$. One solution is to truncate the eigenspace $\*W$. However, this will reduce the number of directions of separability in the data. To retain the full dimensionality of the covariance matrix, we use ``shrinkage" regularization \cite{ledoit2000well} for $\*R_w$ with a parameter $\nu=0.2$ and normalized trace parameter $\bar{\lambda}=\Tr(\*R_w)$ as $\Tilde{\*R}_w = (1-\nu)\*R_w + \nu\bar{\lambda}\*I_d/d$

\textbf{Loss function is bounded}: The objective $\rho_M$ is the average of $d$ eigenvalues obtained by solving GEV. We can analytically show that this objective is bounded above by 1 (See Suppl. methods~S2). Thus, during training, we minimize the loss $1-\rho_M$ to avoid trivial solutions.

\textbf{Inference}: Maximizing $\rho$ leads to maximally correlated embeddings. Thus, during inference we only need to extract embeddings from one of the sub-networks. 
The proposed loss ensures that the different embeddings are maximally correlated (See Suppl. methods simulations~S3).

\subsection{View bootstrapping}
Modeling a large number of views would require many sub-networks which is not practical for hundreds of views. 
To address this issue, we propose view bootstrapping. The schematic of the overall method is shown in Figure~\ref{fig:schematic}.
Here, we construct a network with $m$ sub-networks and sample with replacement a small number of views $m\ll M$ to model data with $M$ views. 
During training, we do not keep track of views being sampled for specific sub-networks which ensures that the model is view-agnostic.
The bootstrapped objective can be written as:
\begin{equation}\label{eq:stochastic}
    \rho^{*} = \mathbb{E}_{m\sim\mathcal{U}(1,M)}\rho_m \approx \rho_M
\end{equation}
The intuition behind our stochastic extension lies in law of large numbers applied to the covariance matrices in Eq.~\ref{eqn:rho-M}.
Let $\*R_{\{b,w\}}$ now denote the covariances estimated from $m$ views. Asymptotically, with a large~$M$ and as~$m\rightarrow M$, we have $\E{\*R_b^{(m)}}\rightarrow \*\Sigma_b$ and $\E{\*R_w^{(m)}}\rightarrow\*\Sigma_w$ where  $\*\Sigma_b\text{ and }\*\Sigma_w$ are the between- and within-view covariance estimated for all $M$ views.
In practice, the number of available view samples is finite and the total number of views possible is often unknown. 
Thus, we analyze the error of the estimate $\rho_m$ with respect to $\rho^*=d^{-1}\Tr(\*W^{\T}\*\Sigma_b\*W)/\Tr(\*W^{\T}\*\Sigma_w\*W)$ in a non-asymptotic setting. 
\begin{theorem}{
Let $\*X=[\*A^{(1)},\dotsc,\*A^{(N)}]$ be the $m\times d$ matrices of $m$ views sampled from an unknown number of views $M$. Let the rows $\*A_l$ of the view matrices $\*A$ be independent subgaussian vectors in $\R^d$ with $\norm{\*A_l}_2=1:l=1,\dotsc,m$. Then for  any $t\geq0$, with probability at least $1-2\exp({-ct^2})$, we have
\begin{flalign*}
{\rho_m}\leq \max\bigg(1, C\frac{m^2}{(\sqrt{d}+t)^2}\rho^{*}\bigg)
\end{flalign*}
\textit{Here, $\rho_m$ and $\rho^*$ are the mv-corr objectives for subsampled views $m$ and the total number of views $M$ respectively. The constant $C$ depends only on the subgaussian norm $K$ of the view space, with $K = \displaystyle\max_{i,l}\norm{\*A_l^{(i)}}_{\psi_2}$}
}\end{theorem}
\begin{psketch}
Here we highlight the main elements of the proof. 
Please see Suppl. methods, Theorem~S6 for the detailed work.
Recall that $\*R_b$ and $\*R_w$ now denote covariance matrices for $m$ views.
Using properties of trace and spectral norm, we can rewrite the expression of the corresponding $\rho_m$ as: 
\begin{align*}
    \rho_m = \frac{\Tr(\*W^{\T}{\*R}_b\*W)}{\Tr(\*W^{\T}{\*R}_w\*W)} 
    = \frac{\<{\*R_b} + \*{\Sigma}_b - \*{\Sigma}_b, \*W\*W^{\T}\>}{\<{\*R}_w + \*{\Sigma}_w - \*{\Sigma}_w, \*W\*W^{\T}\>}\\
    \leq \frac{\<{\*{\Sigma}}_b, \*W\*W^{\T}\> + \norm{{\*R}_t - \*{\Sigma}_t} + \norm{{\*R}_w - \*{\Sigma}_w}}{\<{\*{\Sigma}}_w, \*W\*W^{\T}\> - \norm{{\*R}_w - \*{\Sigma}_w}}
\end{align*}
where $\*{\Sigma}_b$ and $\*{\Sigma}_w$ are the previously defined between- and within-view covariances respectively for $M$ views. 
From Eq.~\ref{eqn:rt}, recall the result: $\*R_b = \*R_t-\*R_w$. The rest follows through triangular inequalities.
Observe that the ratio $\<{\*{\Sigma}}_B, \*W\*W^{\T}\>/\<{\*{\Sigma}}_W, \*W\*W^{\T}\>$ is the optimal $\rho^*$ estimated from the unknown number of views $M$. 
Also, the two trace terms are sum of normalized eigenvalues. Thus $\abs{\<{\*{\Sigma}}_b, \*W\*W^{\T}\>},\abs{\<{\*{\Sigma}}_W, \*W\*W^{\T}\>}\in[1,d]$. 

Next, we need to bound the two norms $\delta_t = \norm{\*R_t-\*\Sigma_t}$ and $\delta_w=\norm{\*R_w-\*\Sigma_w}$. In the statement of the theorem, note that the multi-view data matrix $\*X$ was rearranged as  $[\*A^{(1)},\dotsc,\*A^{(N)}]$ using the features as rows in the view-matrices $\*A$. Thus, using the identicality assumption of multiple views, we have:
{\begin{align*}
\delta_w 
&= \norm{\sum_{i=1}^{N}\frac{1}{m}\*A^{(i)\T}\*A^{(i)} - \E\*A^{(i)\T}\*A^{(i)}} \\\nonumber 
&\leq \sum_{i=1}^{N}\norm{\frac{1}{m}\*A^{(i)\T}\*A^{(i)} - \*{\Sigma}_w^{(i)}}
\leq N\norm{\frac{1}{m}\*A^{\top}\*A - \*\Sigma_w}
\end{align*}}

The term $\norm{\frac{1}{m}\*A^{\T}\*A-\*\Sigma_w}$ has been extensively studied for the case of isotropic distributions i.e., $\*\Sigma_w=\*I$ by \citet{vershynin2010introduction}. Here, we obtain a bound for the general case of $\*\Sigma_w$ and show that $\delta_w=\norm{\*R_w-\*\Sigma_w}$ is $\mathcal{O}(d/m)$.
Similarly, we can show that $\delta_t=\norm{\*R_t-\*\Sigma_t}\leq m$. The intuition here is that $\*R_t$ is sum of $m$ view vectors, hence it is $\mathcal{O}(m)$.
Detailed proofs for $\delta_w$ and $\delta_t$ are provided in Suppl. methods, Lemmas~S4 and S5.
Using these results and the fact that we always choose an embedding dimension $d$ greater than $m$, we can prove that $\rho_m$ is $\mathcal{O}({m^2}/{d})$.
\end{psketch}

This result is significant because we can now show that, to obtain $d$-dimensional multi-view embeddings, we only need to subsample $m\leq\sqrt{d}$ number of views from the larger set of views. For example, for a 64-dimensional embedding, we would need to sample at most $8$ views. In other words, the DNN architecture in this case would have $8$ sub-networks. Additionally, the choice of $d$ is important because a small $d$ would only discriminate between classes that are already easily separable in the data. In contrast, a larger $d$ would require a greater $m$ which in turn inflates the number of parameters in the DNN. 

\section{Experiments}
We conducted experiments with three different datasets to benchmark the performance of our method with respect to the competitive baselines specific to these domains.
We chose these datasets to assess the applicability of our method for downstream learning tasks in two distinct multi-class semi-supervised settings: (1) uniform distribution of views per class and (2) variable number of views per class.

\subsection{3D object classification}\label{sec:3d-object}
We use Princeton ModelNet dataset \cite{wu20153d} to classify the object type from 2D images acquired at multiple view points.
We use the train/test splits for the 40-class subset provided in their website\footnote{3D object dataset and leader-board:\hyperlink{https://modelnet.cs.princeton.edu/}{modelnet.cs.princeton.edu}\label{note1}}.
Each class has $100$ CAD models ($80$/$20$ for train/test) with 2D images ($100\times100$px) rendered in two settings by \citet{su2015multi}: \textbf{V-12}: $12$ views by placing virtual cameras at $30$ degree intervals around the consistent upright position of an object and \textbf{V-80}: $80$ views rendered by placing $20$ cameras pointed towards the object centroid and rotating at $0, 90,
180, 270$ degrees along the axis passing through the
camera and the object centroid.
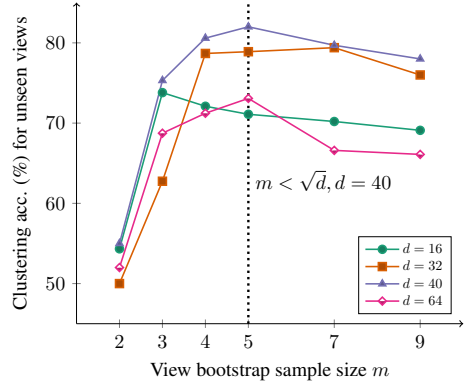
\begin{figure}[t!]
\vskip 0.05in
\begin{center}
\scalebox{0.75}{
\begin{tikzpicture}
\begin{axis}[
     legend style={nodes={scale=0.7, transform shape}}, 
     axis lines=middle,
     axis line style={->},
     x label style={at={(axis description cs:.5,-0.1)},anchor=north},
     y label style={at={(axis description cs:-0.1,.5)},rotate=90,anchor=south},
     xlabel={View bootstrap sample size $m$},
     ylabel={Clustering acc. (\%) for unseen views },
     xmin= 1, xmax= 10,
     ymin= 45, ymax= 85,
     xtick={2,3,4,5,7,9},
     ytick={50, 60, 70, 80, 90},
     legend pos=south east,
     every axis plot/.append style={thick}
]
\addplot
     coordinates {(2, 54.331)(3, 73.800)(4, 72.100)(5, 71.100)(7, 70.200)(9, 69.100)};
\addplot
     coordinates {(2, 50.000)(3, 62.755)(4, 78.680)(5, 78.900)(7, 79.400)(9, 76.000)};
\addplot
     coordinates {(2, 55.000)(3, 75.313)(4, 80.590)(5, 82.000)(7, 79.700)(9, 78.000)};
\addplot
     coordinates {(2, 52.000)(3, 68.741)(4, 71.230)(5, 73.080)(7, 66.600)(9, 66.100)};
\draw [very thick,dotted] (5,\pgfkeysvalueof{/pgfplots/ymin}) -- (5,\pgfkeysvalueof{/pgfplots/ymax});
\node at (axis cs:5,65) [anchor=north west] {$m<\sqrt{d},d=40$};
    \legend{$d=16$,$d=32$,$d=40$,$d=64$};
\end{axis}
\end{tikzpicture}

}
\caption{Clustering accuracy of unseen views for different choices of embedding-dimension $d$ and number of views subsampled $m$}
\label{fig:3d-acc}
\end{center}
\vskip -0.3in
\end{figure}
\subsubsection{Deep mv-corr Model}\label{sec:model}
As shown in Figure~\ref{fig:schematic}, we use identical sub-networks to model the data from each view. 
The number of sub-networks is equal to the number of views subsampled $m$. We use a simple 3-block VGG architecture \cite{chatfield2014return} as illustrated in the inset in Figure~\ref{fig:schematic}. 
To reduce the number of trainable parameters, we use global average pooling after the last layer instead of vectorizing its activations before passing them to a dense layer of $d$ neurons. The embedding layer is constrained to have a unit $l_2$ norm. For all our experiments, we observed that a sigmoid activation for all layers yielded maximum $\rho$ at convergence. 
The loss $1-\rho$ was minimized using SGD with a learning rate of $0.01$, momentum of $0.9$ and a decay of $1e-6$. 
To determine model convergence, we applied early stopping criteria (stop training if $1-\rho$ at the end of a training epoch did not decrease by $10^{-3}$ for $5$ consecutive epochs).
All models were implemented in TensorFlow\footnote{TensorFlow 2.0: \hyperlink{https://www.tensorflow.org/api/r2.0}{tensorflow.org/api/r2.0}} and trained on a GeForce GTX 1080 Ti GPU.

The result in Theorem 3.1 only tells us about the relation between $d$ and $m$ and not their effect on classification accuracy, so we trained models for $m=[2,3,4,5,7,9]$ and $d=[16,32,40,64]$. Note that, during training we only need to know that the $m$ samples per instance in a batch are of the same class, hence the training can be considered semi-supervised. 
During inference, we just extract embeddings from one of the sub-networks which is randomly chosen. We did not observe significant changes in performance by choosing a different sub-network.
\begin{table}[t]
\caption{Accuracy of clustering for seen and unseen views. SD computed from ten trials. Bold indicates significantly higher acc.}
\label{table:3d-object}
\vskip 0.15in
\begin{center}
\begin{small}
\begin{sc}
\begin{tabular}{lcccr}
\toprule
\multicolumn{2}{l}{dataset/model} & Ours           & supervised     \\
\midrule
\multirow{2}{*}{\textbf{V-12}}   & seen    & 82.9 $\pm$ 0.5 & \textbf{88.7} $\pm$ \textbf{1.2} \\
                        & unseen  & 82.1 $\pm$ 0.7 & 81.5 $\pm$ 0.9  \\
\midrule
\multirow{2}{*}{\textbf{V-80}}   & seen    & 84.2 $\pm$ 0.4 & \textbf{89.2} $\pm$ \textbf{1.4} \\
                        & unseen  & \textbf{85.7} $\pm$ \textbf{1.1} & 80.3 $\pm$ 1.5\\
\bottomrule
\end{tabular}
\end{sc}
\end{small}
\end{center}
\vskip -0.4in
\end{table}

\subsubsection{Robustness to unseen views}
To setup a view-agnostic evaluation, of the $80$ CAD models in the ModelNet training data, we pick $6$ views for~\textbf{V-12} and $40$ views for \textbf{V-80} to create a train split. 
We create ten such trials by choosing the 50\% of the views using a different random seed. 
View-information was only used to ensure no overlap of views in train/test splits. 
We then evaluate the performance of our model on the $20$ CAD models in the test-set both for views that were \textit{seen} and \textit{unseen} in training. 

As described in Sec.~\ref{sec:model}, we train our models in a semi-supervised fashion. 
We use k-means algorithm~\cite{scikit-learn} (no. clusters set to $40$) to classify the $40$ classes in the test set.
For baselines, we train a fully supervised CNN in a view-agnostic fashion with same architecture as our sub-network.
This baseline can be considered as an upper bound of performance as it is fully supervised.

First, we examine the clustering accuracy\footnote{Clustering accuracy estimated with Kuhn's Hungarian method\label{kuhn}} for different choices of $m$ and $d$ on the test-set of unseen views in \textbf{V-12}. As shown in Figure~\ref{fig:3d-acc}, we found that $d=40$ with the number of sub-networks $m=5$ gave the best performance. Consistent with our theory, $m>\sqrt{d}$ did not improve the performance further. The dip in performance for $m\geq7$ in this case maybe due to the limited data for larger networks.

Then, we compare the clustering performance of the chosen model on the test set for the views seen and unseen during training, as well as with the supervised baseline. 
As shown in Table~\ref{table:3d-object}, for our method, there is no significant\footnote{Significance testing using Mann-Whitney U test at $\alpha=0.05$} difference between accuracy scores for seen and unseen views for the ten trials. The results for the supervised baseline show significantly better performance for seen views compared to that of unseen views. This suggests that our method performs better for views not in training data. Additionally for the \textbf{V-80} dataset, our model performs significantly better than the supervised baseline, suggesting the benefit of multi-view modeling in case of a denser view sampling.\linebreak
\begin{table}[t]
\caption{3D object recognition and retrieval comparison with other methods. Bold indicates results of the SoA.}
\label{table:3d-soa}
\vskip 0.15in
\begin{center}
\begin{small}
\begin{sc}
\begin{tabular}{lcc}
\toprule
Method                         & Acc. & mAP  \\
\midrule
Loop-view CNN \cite{jiang2019mlvcnn} & 0.94 & \textbf{0.93} \\
HyperGraph NN \cite{feng2019hypergraph}  & \textbf{0.97} & -    \\
Factor GAN \cite{khan2019unsupervised}     & 0.86 & -    \\
MVCNN \cite{su2015multi}            & 0.90 & 0.80 \\
Ours + 3-layer DNN             & 0.94 & 0.89 \\
\bottomrule
\end{tabular}
\end{sc}
\end{small}
\end{center}
\vskip -0.3in
\end{table}
\subsubsection{Object recognition and retrieval}\label{sec:supervised}
To evaluate our embeddings in a supervised setup, we train a model as described in sec.~\ref{sec:model} using $40$ CAD models in the train split. We extract the embeddings for the remaining $40$ CAD models and train a 3-layer fully connected (sigmoid activation) DNN to classify the object category.
We use classification accuracy and mean average precision (mAP) to evaluate recognition and retrieval tasks.
For baselines, we compare our method with the ModelNet leader-board\footref{note1} for \textbf{V-12}. 
We highlight our results in Table~\ref{table:3d-soa} in the context of state of the art (SoA) performance for this application as well as examples from widely used class of methods such as domain-invariant applications of GAN \cite{khan2019unsupervised} and multi-view CNN for object recognition \cite{sun2019multi}. Unlike our method, these methods are fully supervised and are not generally view-agnostic.

Our method performs within 4\% points of the SoA for recognition and retrieval tasks (See Table~\ref{table:3d-soa}).
In all experiments, we observed that the bound for maximum number of sub-networks is better in practice than the theoretical bound, i.e. $m\approx d^{2/5}$.
Also, the choice of $m$ only varied with $d$ and not the larger set of views $M$ which is a useful property to note for practical settings. 
The parameter $d$ however needs to be tuned for classification tasks as it depends on intra- and inter-class variabilities which determine the complexity of the downstream task.

\subsection{Pose-invariant face recognition}
\label{sec:pose}
Robust face recognition is yet another application where multi-view learning solutions are attractive because we are interested in the shared representation across different presentations of a person's face. For this task, we use the Multi-PIE face database \cite{gross2010multi} which includes face images of $337$ subjects in $15$ different poses, $20$ lighting conditions and $6$ expressions across 4 sessions. 

In Sec.~\ref{sec:3d-object}, we evaluated our model to classify object categories available for training, but with a focus on the performance of seen vs. unseen views during training. 
In this experiment, we wish to test the usefulness of our embeddings to recognize faces not seen in training.
We use a similar train/test split as in GMA \cite{sharma2012generalized} of $129$ subjects in $5$ lighting conditions ($1, 4, 7, 12, 17$) common to all four sessions as test data and remaining $120$ subjects in session 01 for training.
For performance evaluation, we use 1-NN matching with normalized euclidean distance similarity score as the metric.
The gallery consisted of faces images of the $129$ individuals in frontal pose and frontal lighting and the remaining images from all poses and lighting conditions were used as probes. All images were cropped to contain only the face and resized to $100\times100$ pixels. No face alignment was performed.

\subsubsection{Model and baselines}
For our model architecture, we first choose~$m=2$ sub-networks and examine the mv-corr value at convergence for different embedding dimension $d$. Based on this we pick $d=64$. Following our observations in the object classification task, we choose $m=4$ sub-networks. 
The sub-network architecture is the same as before (See inset Figure~\ref{fig:schematic}). 
We did not explore other architectures because our goal here was to evaluate the use of mv-corr loss and not necessarily the best performing model for a specific task. During training, we sample with replacement, $m$ face images per individual agnostic to the pose or lighting condition. For matching experiments, we extract embeddings from a single randomly chosen sub-network.

For baselines, we train deep CCA (DCCA~\citealt{andrew2013deep}) using its implementation\footnote{Deep-CCA code: \hyperlink{https://github.com/VahidooX/DeepCCA}{github.com/VahidooX/DeepCCA}} with the same sub-network architecture as ours.
We trained separate DCCA models for five poses: 15, 30, 45, 60 and 75 degrees. 
While training the two sub-networks in DCCA, we sample face images of subjects across all lighting conditions with a frontal pose for one sub-network and images of specific pose for the second. 
This matches the testing conditions where we only have frontal pose images in the gallery.
During testing we use the pose-specific sub-network to extract embeddings. We also compare our method with two other variants of GMA: GMLDA and GMMFA reported by \citet{sharma2012generalized}.

As shown in Table~\ref{table:face-pose}, our model successfully matches at least 90\% of the probe images to the frontal faces in the gallery, across all poses. The performance drop across different poses was also minimal compared to a pairwise method such as DCCA which assumes that the pose of a probe image  is available in testing conditions.
However, the view-agnostic benefit of our method and the Multi-PIE dataset needs to be viewed in the context of the broader research domain of face recognition. 
Methods such as MvDA \cite{kan2015multi} which build view-specific transformations have shown nearly 100\% face recognition rate on Multi-PIE when the pose information of the probe and gallery images was known.
Furthermore, the face images in this dataset were acquired in strictly controlled conditions. 
While it serves as an effective test-bed for benchmarking, we must consider other sources of noise for robust face recognition besides pose and lighting \cite{wang2018devil}. Our future work will focus on adapting mv-corr for face recognition in-the-wild.
\begin{table}[t]
\caption{1-NN matching accuracy comparison for pose-invariant face recognition. Bold indicates the best performing model}
\label{table:face-pose}
\vskip 0.15in
\begin{center}
\begin{small}
\begin{sc}
\begin{tabular}{lllllll}
\toprule
Method                       & $15^{\circ}$   & $30^{\circ}$   & $45^{\circ}$   & $60^{\circ}$   & $75^{\circ}$   & Avg.   \\
\midrule
GMLDA                       & 92.6 & 80.9 & 64.4 & 32.3 & 28.4 & 59.7  \\
GMMFA                       & 92.7 & 81.1 & 64.7 & 32.6 & 28.6 & 59.9  \\
DCCA                     & 82.4 & 79.5 & 73.2 & 62.3 & 51.7 & 69.8 \\
Ours & 95.7 & 93.1 & 94.5 & 92.3 & 91.1 & \textbf{93.3} \\
\bottomrule
\end{tabular}
\end{sc}
\end{small}
\end{center}
\vskip -0.3in
\end{table}

\subsection{Spoken word recognition}
\label{speech}
The multi-view datasets considered in sections~\ref{sec:3d-object} and \ref{sec:pose} for benchmarking our method were acquired in controlled conditions. They also have nearly uniform distribution of number of distinct views per class as well as as uniform number of samples per view.
In practical settings, we often have to deal with a variable number of views per class. To study the our framework in this context, we evaluate our method for spoken word recognition using the publicly available Speech Commands Dataset (SCD, \citealt{warden2018speech}).

\subsubsection{Speech Commands Dataset}
SCD includes variable number of one second audio. recordings from over 1800 speakers saying one or more of 30 commands such as ``On'' and ``Off''.  
The application of mv-corr for spoken-word recognition and text-dependent speaker recognition in SCD was studied by~\citet{somandepalli2019multiview}. Their results showed improved performance for speaker recognition task compared to the SoA in this domain~\cite{snyder2017deep}. 
Building upon their work, in this paper, we analyze spoken-word recognition on SCD in a greater detail. 

The different speakers saying the same word can be treated as multiple views to obtain discriminative embeddings of the speech commands invariant to the speaker (view).
Specifically, we are interested in the performance of our method for the case of variable number of views per class. Thus, we analyze the performance of each word with respect to the number of unique speakers (views) available for that word.
\linebreak
\newline
We choose $m=4$ sub-networks (See inset Figure~\ref{fig:schematic} for the architecture) to obtain 64-dimensional embeddings. 
Of the $1868$ speakers, we use $1000$ speakers for training and the remaining for testing to ensure that we only test on speakers (views) not seen during training. 
To assess generalizability to unseen classes, we create three folds by including $20$ words for training and the remaining $10$ words for testing. 
The models are trained in a semi-supervised fashion as described in~\ref{sec:model}.
We use the k-means algorithm to cluster the embeddings for the test splits with the number of clusters set to $10$. 

The per-class accuracy\footref{kuhn} from the clustering task is shown in Figure~\ref{fig:scd}. 
The average number of speakers across the thirty commands was $400.3\pm52.5$ which underscores the variable number of views per class.
We observe a minimal association (Spearman rank correlation = $0.12$) between the number of unique speakers per word and the per-class accuracy scores.
However, it is difficult to disambiguate this result from the complexity of the downstream learning task. That is, we may need more views for certain words to account for inter-class variability (similar sounding words e.g., ``on'' vs. ``off'' or ``tree'' vs. ``three'') and intra-class variability (e.g., different pronunciations of the word ``on'').

\begin{figure}[t!]
\vskip 0.05in
\begin{center}
\scalebox{0.6}{\begin{tikzpicture}
\begin{axis}[symbolic x coords={dog,
go,
tree,
left,
on,
off,
up,
three,
cat,
bed,
seven,
eight,
four,
five,
right,
two,
down,
stop,
six,
nine,
yes,
house,
no,
bird,
one,
marvin,
sheila,
zero,
wow,
happy},
xtick={dog,
go,
tree,
left,
on,
off,
up,
three,
cat,
bed,
seven,
eight,
four,
five,
right,
two,
down,
stop,
six,
nine,
yes,
house,
no,
bird,
one,
marvin,
sheila,
zero,
wow,
happy},  
xticklabel style={text height=0ex,
tick label style={rotate=90}}, 
ymin= 0.3, ymax= 0.85,
ytick={0.4, 0.5, 0.6, 0.7, 0.8},
width=0.7\textwidth, 
bar width=1pt,
xtick style={draw=none},
ymajorgrids=true,
ylabel={Clustering accuracy (acc.) per speech command},
] 

\addplot[ybar,fill] coordinates {
(dog,0.420)
(go,0.437)
(tree,0.471)
(left,0.510)
(on,0.538)
(off,0.549)
(up,0.581)
(three,0.582)
(cat,0.603)
(bed,0.632)
(seven,0.637)
(eight,0.653)
(four,0.667)
(five,0.670)
(right,0.683)
(two,0.683)
(down,0.685)
(stop,0.690)
(six,0.703)
(nine,0.704)
(yes,0.713)
(house,0.714)
(no,0.717)
(bird,0.722)
(one,0.760)
(marvin,0.778)
(sheila,0.797)
(zero,0.820)
(wow,0.838)
(happy,0.839)
};
\coordinate (A) at (axis cs:dog,0.659);
\coordinate (O1) at (rel axis cs:0,0);
\coordinate (O2) at (rel axis cs:1,0);
\draw [black,sharp plot,dashed] (A -| O1) -- (A -| O2);
\node at (axis cs:go,.75) [anchor=north west] {Average class acc. = 0.66};
\end{axis}
\end{tikzpicture}
}
\caption{Per-class clustering accuracy of mv-corr method for spoken word recognition in SCD}
\label{fig:scd}
\end{center}
\vskip -0.3in
\end{figure}
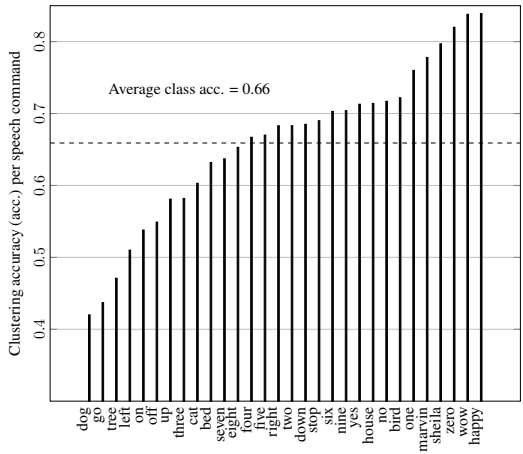

\subsubsection{Domain adversarial learning}\vspace{-0.05in}
Finally, in the context of domain adaptation for experiments with SCD, we compare our multi-view learning method with two recent domain adversarial learning methods: domain adversarial networks (DAN,~\citealt{ganin2016domain}) and cross-gradient training (CrossGrad,~\citealt{shankar2018generalizing}). The central idea of these methods is to achieve domain invariance by training models to perform better at classifying a label than at classifying the domain (view).
\begin{table}[t]
\caption{Comparison of mv-corr framework with domain adversarial methods}
\vspace{-0.15in}
\label{table:scd}
\vskip 0.15in
\begin{center}
\begin{small}
\begin{sc}
\begin{tabular}{lccc}
\toprule
Method   & DAN  & CrossGrad & Ours + 2-layer DNN \\
\midrule
Acc (\%) & 77.9 & 89.7      &\textbf{ 92.4}  \\
\bottomrule
\end{tabular}
\end{sc}
\end{small}
\end{center}
\vskip -0.3in
\end{table}

As described in Sec.~\ref{sec:supervised}, we adapt our embeddings for a supervised setting on a subset of $12$ commands in SCD to compare with the results in \citet{sharma2012generalized}. 
We first train the mv-corr model of four sub-networks using $500$ speakers from the training set.
We then obtain 64-dimensional embeddings on the remaining $500$ speakers and train a 2-layer fully connected DNN (sigmoid activation) to classify the $12$ commands, and test on the remaining $868$ speakers.
For baselines, we replicate the experiments for DAN and CrossGrad using released code.\footnote{CrossGrad and DAN code: \hyperlink{https://github.com/vihari/crossgrad}{github.com/vihari/crossgrad}}
We use the same splits of $500$ speakers each for training/development and $868$ speakers for testing.
The classification accuracy of our method and that of DAN and CrossGrad is shown in Table~\ref{table:scd}. 
We observed a significant improvement\footnote{Permutation test $n=10^5$, $p = 0.008$} over CrossGrad, suggesting that a multi-view formulation can be effectively used for domain adaptation problems such as in SCD.

\section{Conclusion}\vspace{-0.05in}
In this paper, we explored a neural method based on multi-view correlation (mv-corr) to capture the information shared across large number of views by bootstrapping the data from multiple views during training in a view-agnostic manner.
We discussed theoretical guarantees of view bootstrapping as applied to mv-corr and derived an upper bound for the number of views to subsample for a given embedding dimension.
Our experiments on 3D object classification and retrieval, pose-invariant face recognition and spoken word recognition showed that our approach performs on par with competitive methods in the respective domains.
Our results underscore the applicability of our framework for large-scale practical applications of multi-view data where we may not know how the multiple corresponding views were acquired.
In future work, we wish to extend the ideas of view-bootstrapping and related theoretical analysis to the broader class of multi-view learning problems.




\onecolumn
\icmltitle{Generalized Multi-view Shared Subspace Learning using View Bootstrapping}
\icmlsetsymbol{equal}{*}


\renewcommand{\thesection}{S\arabic{section}}  

\section*{Supplementary Methods}
The following sections provide detailed proofs for propositions, lemmas and the theorem presented in the associated ICML submission.
We also provide details of simulation analysis that we conducted to support one of the claims made in the paper.
\tableofcontents
\vspace{0.1in}
\begin{table}[htbp]
\begin{tabular}{l | p{5cm} }
\toprule
Section & Link \\
\midrule\\
Table of Notations & \ref{table:notation}\\[0.2cm]
Proposition: Total-view Covariance & \ref{prop:rt}\\[0.2cm]
Proposition: Multi-view correlation objective is bounded above by 1 & \ref{prop:rho}\\[0.2cm]
Simulation Experiments & \ref{sec:supp-affinity}\\[0.2cm]
Lemma: Upper Bound for Bootstrapped Within-View Covariance & \ref{sec-lemma:rw} \\[0.2cm]
Lemma: Upper Bound for Bootstrapped Total-View Covariance & \ref{sec-lemma:rt} \\[0.2cm]
Theorem: Error of the Bootstrapped Multi-view Correlation & \ref{thm:rho} \\[0.2cm]
\bottomrule
\end{tabular}
\end{table}

\newpage
\section*{Notation}
\begin{table}[htbp]\caption{Notations used in the proofs and text}
\label{table:notation}
\begin{center}
\begin{tabular}{r | p{10cm} }
\toprule
$N$ & Number of samples \\
$M$ & Number of views \\
$d$ & Embedding dimension\\
$m$ & Bootstrap view sample size / number of subsampled views\\
$\*x_i\in\R^d$ & Embedding/feature vector\\
$i=1,\dotsc,N$ & Index for sample \\
$l=1,\dotsc,M$ & Index for view \\
$\*X_l=[\*x_1,...,\*x_N] \in \R^{d\times N}$ & $l^{th}$-view data matrix\\
$\*X=[\*X_1,\dotsc,\*X_M]$ & Multi-view data matrix. Assume mean-zero columns without loss of generality\\
$\*R_b\in\R^{d\times d}$ & Sum of between-view covariance matrices for $m$ views: \textit{Between-view covariance}\\
$\*R_w\in\R^{d\times d}$ & Sum of within-view covariance matrix : \textit{Within-view covariance}\\
$\*R_t\in\R^{d\times d}$ & Total-view covariance matrix \\
$\*\Sigma_b$ & Between-view covariance for $M$ views \\
$\*\Sigma_w$ & Within-view covariance for $M$ views \\
$\*\Sigma_t$ & Total-view covariance for $M$ views \\
$\*A_l\in\R^d$ & $d$-dimensional feature row, mean-zero and $\norm{\*A_l}_2=1$\\
$\*A^{(i)}=[\*A_1;\dotsc;\*A_m]\in\R^{m\times d}$ & View-matrix from the $i^{th}$ sample for $m$ views with features as rows\\
$\*X = [\*A^{(1)},\dotsc,\*A^{(N)}]$ & Rearranged m-view data matrix \\
$\*W \in\R^{d\times d}$ & Shared subspace / Common Eigenspace of $\*R_b$ and $\*R_w$   \\
$\norm{\cdot}_2 \equiv \norm{\cdot}$ & Spectral norm \\
$\norm{\cdot}_{\psi_1}$  & Sub-exponential norm \\
$\norm{\cdot}_{\psi_2}$ & Sub-gaussian norm \\
\bottomrule
\end{tabular}
\end{center}
\label{tab:TableOfNotationForMyResearch}
\end{table}

\newpage
\section{Proposition: Total-view Covariance}\label{prop:rt}
Consider the sum of $\*R_b$ and $\*R_w$ which includes $M^2$ terms. Note that we assume $\*X_l:l=1,\dotsc,M$ to have mean-zero columns. Therefore covariance estimation is just the cross-product:
\begin{align*}\label{eq:rt}
    \*R_w + \*R_b 
    &= \frac{1}{M}\sum_{l=1}^{M}{\*X}_{l}({\*X}_{l})^{\top} + \frac{1}{M}\sum_{k=1}^{M}\sum_{l=1,l\neq k}^{M}{\*X}_{l}({\*X}_{k})^{\top} &&\text{[By definition]} \\\nonumber
    &= \frac{1}{M}\sum_{l=1}^{M}\sum_{k=1}^{M}{\*X}_{l}({\*X}_{k})^{\top} &&\text{[Summing all terms]}\\\nonumber
    &= \frac{1}{M}\bigg(\sum_{l=1}^{M}{\*X}_{l}\bigg)\bigg(\sum_{l=1}^{M}{\*X}_{l}\bigg)^{\top} =\*R_t &&\text{[Total-view covariance]}\\\nonumber
\end{align*}
where the total-view matrix is $\sum_{j=1}^{M}\*X_j$.
Thus, $\*R_t$ can be easily estimated as the covariance of a single total-view matrix, without having to consider the sum of $M^2-M$ covariance matrices. 
Note that we excluded the normalization factor $N-1$ in the esimtation of the covariance terms above.
This gives us the following useful relation which simplifies many computations in practice.
\begin{align}
    \*R_t &= \*R_b + \*R_w
\end{align}

\newpage
\section{Proposition: Multi-view correlation objective is bounded above by 1}\label{prop:rho}

Recall the multi-view correlation objective for $M$ views:
\begin{align}\label{eqn:rho-M}
    \rho_{M}
    = \max_{\*W}\frac{1}{d(M-1)}\frac{\Tr(\*W^{\T}\*R_b\*W)}{\Tr(\*W^{\T}\*R_w\*W)} 
\end{align}

It is desirable to have an upper bound for the objective similar to the correlation coefficient metric which is normalized to have a maximum value of 1.
 Let us begin with the definition of the multi-view correlation matrix:
 \begin{eqnarray}\label{eqn:mv-corr}
    \*\Lambda = \max_{\*W}\frac{1}{M-1}\frac{\*W^{\T}\*R_b\*W}{\*W^{\T}\*R_w\*W}
\end{eqnarray}
Here, $\*W\in\R^{d\times M}\,,M\leq d$

Define a matrix $\*Y_l = \*W^{\T}\*X_l\in\R^{M\times N}\,,M\leq d$ where the column vectors $\*y\in\R^M$ are a low-dimensional projection of the input features $\*X$. The column vector elements are $y_i^l\in\R:i=1,\dotsc,N;l=1,\dotsc,M$ with that the ratio in Eq.~\ref{eqn:mv-corr}, ignoring the max operation can be written as:

\begin{align*}
    \*\Lambda &= \frac{1}{M-1}\frac{\*W^{\T}(\*X_1\*X_2^{\T}+\dotsc+\*X_M\*X_{M-1}^{\T})\*W}{\*W^{\T}(\*X_1\*X_1^{\T}+\dotsc+\*X_M\*X_M^{\T})\*W} \\
    &= \frac{1}{M-1}\frac{(\*Y_1\*Y_2^{\T}+\dotsc+\*Y_M\*Y_{M-1}^{\T})}{(\*Y_1\*Y_1^{\T}+\dotsc+\*Y_M\*Y_M^{\T})} \\
    &= \frac{1}{M-1}\frac{\sum_i\sum_l\sum_{k\neq l}y_i^ly_i^k}{\sum_i\sum_l (y_i^l)^2}\\
    &= \frac{1}{M-1}\frac{r_b}{r_w}
\end{align*}

To show that $\rho\leq1$, we can also equivalently prove the following expression is non-negative:

\begin{align*}
    0 &\leq (M-1)r_w - r_b 
    = (M-1)r_w - (r_t - r_w)  &&\text{[From total-covariance proposition: Sec.\ref{prop:rt}]} \\
    &= Mr_w - r_t 
    = M\sum_i\sum_l(y_i^l)^2 - \sum_i\big(\sum_l y_i^l\big)^2
    := F
\end{align*}

Now, we need to find the $y_i^l$ that minimizes $F$. Therefore, take the gradient of $F$ with respect to $y_i^l$ and check if the curvature is non-negative where the gradient is zero.

\begin{align}
    \pdv{F}{y_i^l} 
    &= 2My_i^l - 2\sum_j\sum_ly_j^k\sum_l\delta_{ji}^{kl}
    = 2My_i^l - 2\sum_ky_j^k\\
    \pdv[2]{F}{y_i^l}{y_j^k} &= 2M\delta^{lk}_{ij} - 2\sum_t\delta_{ji}^{jt} 
    = 2\delta_{ji}(M\delta^{lk}-1) := J
\end{align}

Solving for $\pdv{F}{y}=0$ has a unique solution: $y_i^l = \frac{1}{M}\sum_k y_i^k = \Bar{y}_i^*$. Putting this result back gives $F=0$ at this solution.
To show this solution minimizes $F$ and therefore $\rho<1$, we need to show that the Jacobian $J$ in Eq.~5 has only non-negative eigenvalues. Note that there are only $\delta$ variables in Eq.~5. Thus, in a matrix form across all views we have $J=M\*I_M - \*I_M$ yielding non-negative eigenvalues. Hence $\rho\leq1$

\newpage
\section{Simulation Experiments}\label{sec:supp-affinity}
In order to show that the output embeddings from the sub-networks are maximally correlated. we need to empirically show that mv-corr is learning highly correlated vector representations.
For this, we generate synthetic observations as detailed in \cite{parra2018correlated} where the number of common signal components across the different views is known. 
Because the source signal is given, we can also empirically examine the correlation of the shared components with the source signal.

\subsection{Data generation}
Consider $N$ samples of signal and noise components for $M$ views to be $\*s^l_n\in\mathbb{R}^K$ and $\*b^l_n\in\mathbb{R}^D\, n=1,...,N\,, l=1,...,M\,, K<D$ respectively, both drawn from standard normal distribution. 
Because our objective is to obtain correlated components across the views, we fixed the same signal component across the $M$ views, i.e, $\*s^l_n \approx \*s_n$, but corrupted with a view-specific noise$\ \*{\eta^l}$.  
Thus, signals were mapped to the measurement space as$\ \*x_{s,n}^l=\*A_s^l\*s_n + \*{\eta}^l, \*x_{b,n}^l=\*A_b^l\*b^l_n$ and were z-normalized. 
The multiplicative noise matrices were generated as $\*A_s^l = \*O_s^l\*D_s^l \in \mathbb{R}^{D\times K}$ and $\*A_b^l=\*O_b^l\*D_b^l \in \mathbb{R}^{D\times D}$ The two matrices $\*O_s^l\in\mathbb{R}^{D\times K}\text{ and }\*O_b^l\in\mathbb{R}^{D\times D}$ are composed of orthonormal columns. 

The non-zero eigenvalues of the signal and noise covariance matrices were set with$\ \*D_s^l\in\mathbf{R}^{K\times K}$ and $\*D_b^l\in\mathbf{R}^{D\times D}$ by constructing$\ D_{ii}=\exp(d_i), d_i\sim\mathcal{N}(0,1)$. 
We used different matrices $\*A_s^l$ and $\*A_b^l$ to simulate a case where the different views of the underlying signal are corrupted by different noise. 
As is the case with many real world datasets, the noise in the measurement signal is further correlated between the views. 
We simulated this by $\*x_{b,t}^l\gets \alpha\*x_{b,n}^l + (1-\alpha)\*x_{b,n}^l , \alpha \in [0,1]$. 
Finally the SNR of the measurements is controlled by $\beta$ to generate the multiview data as $\*y^l_n = \beta\*x_{s,n}^l + (1-\beta)\*x_{b,n}^l , \beta \in [0,1]$ resulting in a data matrix of size $\ N\times\ D\times M$ with $N$ samples of $D$-dimensional data from $M$ views.
For all our experiments, we generated data with $N=100000, D=1024, K=10, M=4, \beta = 0.7$ and spatial noise correlation $\alpha=0.5$.

\subsection{Deep mv-corr Model}
The network consists of 4 sub-networks where each sub-network is composed of 2 fully connected layers of 1024 and 512 nodes which is then fed into an embedding layer with $d=[5, 10, 15, 20, 40, 50, 64, 128]$ neurons. 
The output embedding dimension was varied in order to examine the affinity of the representations with the source signal. This is important, since in real world applications the number of correlated components is not known apriori.
The models were trained as explained in the main paper.

\subsection{Affinity metrics to measure correlation}
The benefit of using synthetic data is that we can examine what the network learns when the generative process is known. 
The affinity measures we use enable us to compare the similarity of the embedding subspaces to that of the source signal.
The objective of our simulations is to measure if the correlated signal components are correctly identified from the measurements. 
Because the components with equal $\rho$ can be produced by arbitrary linear combination of the vectors in the corresponding subspace, we examined the normalized affinity measure between two subspaces as defined in \cite{soltanolkotabi2014robust} to compare the representations with the source signal. 
Let $\ \Hat{\*X}_s^l \in \mathbf{R}^{T\times K'}$ be the reconstructed signal or the representation learnt by optimizing eqn. 11 corresponding to the source signal $\ \*X_s^l \in \mathbf{R}^{T\times K}$. The affinity between $\ \Hat{\*X}\text{ and }\*X$ can be estimated using the principal angles$\ \theta^{(\cdot)}$ as:
\begin{equation}
    \text{aff}(\*X, \Hat{\*X}) = \sqrt{\frac{\cos^2\theta^{(1)} + ... +\cos^2\theta^{(K\wedge K')} }{K\wedge K'}}
\end{equation}
The cosine of the principal angles $\ \theta$ are the singular values of the matrix $\ \*U^{\top}\*V$ where$\ \*U$ and $\ \*V$ are the orthonormal bases for $\ \*X \text{ and } \Hat{\*X}$ respectively. The affinity is a measure of correlation between subspaces and has been extensively used to compare distance between subspaces in the subspace clustering literature \cite{soltanolkotabi2014robust}. This measure is low when the principal angles are nearly orthogonal and has a maximum value equal to one when one of the subspaces is contained in the other.

One of the benefits of using this affinity measure is that it allows us to compare two subspaces of different dimensions. 
We estimate two affinity measures: 1) \textit{reconstruction affinity}, $\ R_a$: average affinity between the reconstructed signal and the source signal across the$\ N$ views and 2) \textit{inter-set affinity}, $\ R_s$: average affinity between the different views of the reconstructed signal. Formally,
{
\begin{eqnarray}
    R_a = \frac{1}{N}\sum_{l=1}^{N}\text{aff}(\*X_s^l,\Hat{\*X}_s^l) \\
    R_s = \frac{2}{N(N-1)}\sum_{l=1}^{N}\sum_{{\substack{k=1 \\ l\neq k}}}^{N}\text{aff}(\Hat{\*X}_s^l,\Hat{\*X}_s^k) 
\end{eqnarray}}

Figure \ref{fig:aff} shows the reconstruction affinity measure ($R_a$) and the inter-set affinity measure ($R_s$) for these parameters.
Notice that the maximum $R_a$ is achieved for the embedding dimension of 10 (which is the number of correlated components used to generate the data) indicating that the dMCCA retains some notion of the ambient dimension for maximizing correlation between views.
The $R_s$ measure consistently decreased with increasing embedding dimension.
Because we estimate covariances in the loss function and use SGD with mini-batches for optimization, we also examine the performance with varying batch sizes. As shown in Fig. \ref{fig:aff} a mini-batch size greater than 400 gives consistent results. The results from this simulation study suggests that the multi-view embeddings are maximally correlated. Hence during inference we can use any sub-network to extract the embeddings.

\begin{figure}[t]
    \centering
    \includegraphics[width=0.8\textwidth]{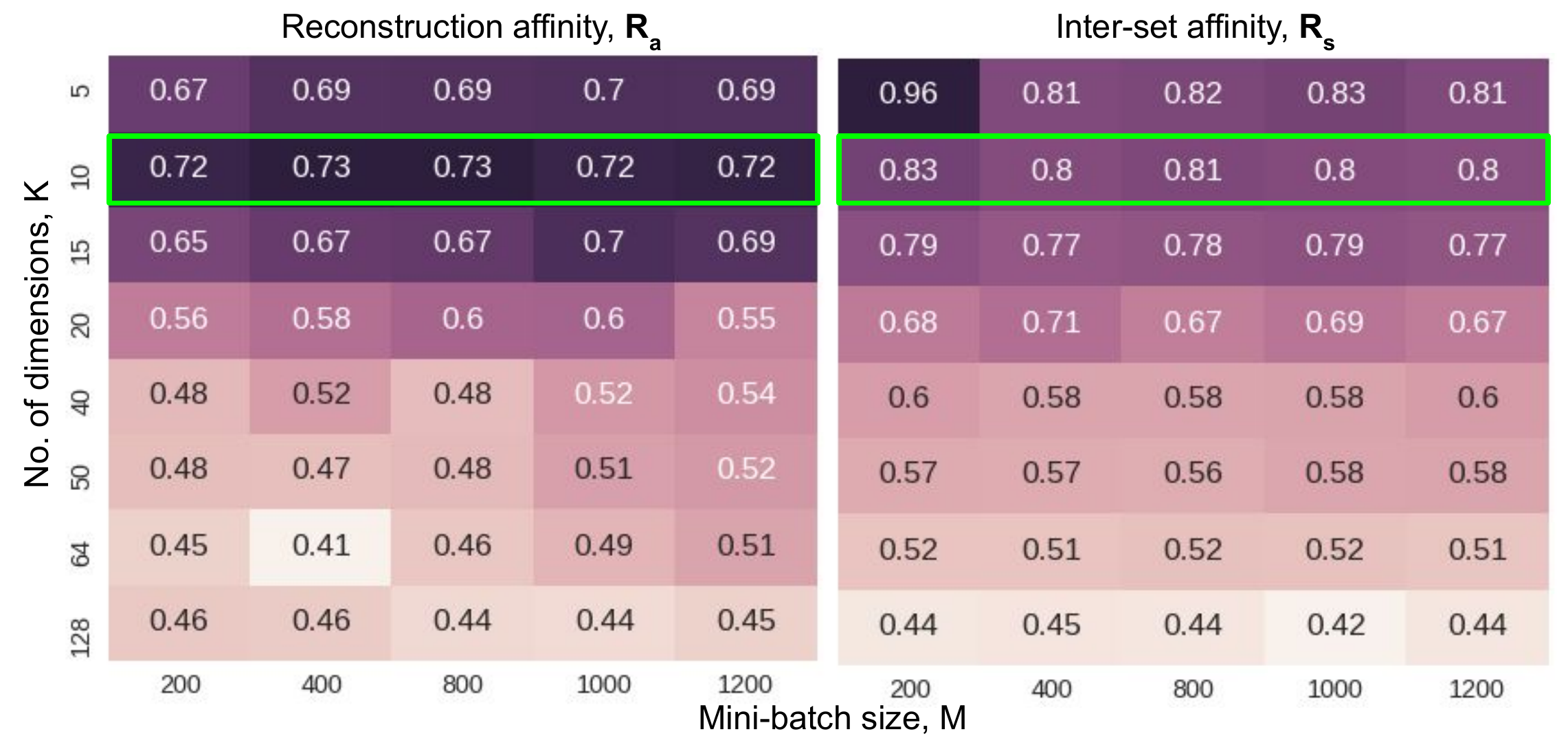}
    \caption{{{Affinity measures for synthetic data. Number of correlated components in the generated data is 10 (boxed)}}}
    \label{fig:aff}
\end{figure}

\newpage
\section{Lemma: Upper Bound for Bootstrapped Within-View Covariance}\label{sec-lemma:rw}

\begin{lemma}\label{lemma:view-matrix}
(Subsampled view matrices, approximate isotropy) Let $\*A$ be a $m\times d$ matrix created by subsampling $m$ views from a larger, unknown number of views. The rows $\*A_i$ of the matrix $\*A$ are independent subgaussian random vectors in $\mathbb{R}^d$ 
and a second moment matrix $\*{\Sigma}=\mathbb{E}\*A_i\otimes\*A_i$. Then for every $t\geq 0$, with probability at least $1-2\exp(-ct^2)$ we have
\begin{align}\label{eq:view-matrix}
    \norm{\frac{1}{m}\*A^{\top}\*A - \*{\Sigma}} \leq \max(\delta, \delta^2) \quad where \quad \delta = C\sqrt{\frac{d}{m}} + \frac{t}{\sqrt{m}}
\end{align}
Here $C, c >0$ depend only on the subgaussian norm $K = \max_i\norm{\*A_i}_{\psi_2}$ of the view space
\end{lemma}
\begin{proof}
This is a straight-forward generalization of Theorem 5.39~\cite{vershynin2010introduction}  for non-isotropic spaces. 
The proof involves \textit{covering argument} which uses a net $\mathcal{N}$ to discretize the compact view space, which is all the vectors $\*z$ in a unit sphere $\mathcal{S}^{d-1}$. Similar to \cite{vershynin2010introduction}, we prove this in three steps:
\begin{enumerate}
    \item \textbf{$\mathcal{N}_{\epsilon}$ Approximation:} Bound the norm $\norm{\*A\*z}_2$ for all $\*z\in\mathbb{R}^d$ s.t. $\norm{\*z}_2=1$ by discretizing the sphere with a 1/4-net.
    \item \textbf{Concentration:} Fix a vector $\*z$, and derive a tight bound of $\norm{\*A\*z}_2$.
    \item \textbf{Union bound:} Take a union bound for all the $\*z$ in the net
\end{enumerate}
\textbf{Step 1: $\mathcal{N}_{\epsilon}$ Approximation.} From \cite{vershynin2010introduction}, we use the following statement:
\begin{align}
    \exists \delta > 0,\quad \norm{\*B^{\top}\*B - \*I} \leq max(\delta, \delta^2) \implies \norm{\*B}_2 \leq 1 + \delta 
\end{align}
We evaluate the operator norm in eq.~\ref{eq:view-matrix} as follows:
\begin{flalign*}
    \norm{\frac{1}{m}\*A^{\top}\*A - \*{\Sigma}}  &= \norm{\frac{1}{m}\*A^{\top}\*A - \frac{1}{m}\mathbb{E}\*A^{\top}\*A} \\
    &= 
    \norm{\frac{1}{m}\Sigma_{i=1}^{m}\*A_i\*A_i^{\top} - \frac{1}{m}\Sigma_{i=1}^{m}\mathbb{E}\*A_i\*A_i^{\top}}
\end{flalign*}
Let $\*D:=\frac{1}{m}\sum_{i=1}^{m}\*A_i\*A_i^{\top} - \frac{1}{m}\Sigma_{i=1}^{m}\mathbb{E}\*A_i\*A_i^{\top}$. Choose a $\epsilon'$-net $\mathcal{N}$ such that $\lvert\mathcal{N}\rvert\leq 9^{d}$ which provides sufficient coverage for the unit sphere $\mathcal{S}^{d-1}$ at $\epsilon'=1/4$. Then, for every $\*z\in\mathcal{N}$ we have (using Lemma 5.4 in \cite{vershynin2010introduction}),
\begin{flalign*}
    \norm{\*D} 
    &\leq {\max_{\substack{\*z\in\mathcal{N} \\ \norm{\*z}=1}}}\abs*{\<\*D\*z, \*z\>} \\
    &\leq \frac{1}{1-2\epsilon'}\max_{\substack{\*x\in\mathcal{N} \\ \norm{\*z}=1}}\norm{\*z^{\top}\*D\*z} \\
    &\leq 2\max_{\*z\in\mathcal{N}}\norm{\*z^{\top}\*D\*z}
\end{flalign*}
For some $\epsilon>0$, we want to show that the operator norm of $\*D$ is concentrated as
\begin{align}
    \max_{\*z\in\mathcal{N}}\norm{\*z^{\top}\*D\*z} \leq \frac{\epsilon}{2} \text{ where } \epsilon:=\max(\delta, \delta^2)
\end{align}
\textbf{Step 2: Concentration.} Fix any vector $\*z\in\mathcal{S}^{d-1}$ and define $Y_i = \*A_i^{\top}\*z - \mathbb{E}\*A_i^{\top}\*z$ where $\*A_i$ are subgaussian random vectors by assumption with $\norm{\*A_i}_{\psi_2}=K$. 
Thus, $Y_i$ $i=1,\dotsc,m$ are independent subgaussian random variables.
The subgaussian norm of $Y_i$ is calculated as,
\begin{align}
    \norm{Y_i}_{\psi_2} = \norm{\*A_i^{\top}\*z - \mathbb{E}\*A_i^{\top}\*z }_{\psi_2} 
    \leq 2\norm{\*A_i^{\top}\*z}_{\psi_2} 
    \leq 2\norm{\*A_i}_{\psi_2}\norm{\*z} 
    = 2K
\end{align}
The above relation is an application of triangular and Jensen's inequalities: $\norm{X-\mathbb{E}X} \leq 2\norm{X}$ with $\abs*{\mathbb{E}X}\leq\mathbb{E}\abs*{X} \leq \abs*{X}$. 
Similarly, $Y_i^2$ are independent subexponential random variables with the subexponential norm $K_e = \norm{Y_i}_{\psi_1} \leq \norm{Y_i}_{\psi_2}^2 \leq 4K^2$. Finally, by definition of $Y_i$, we have
\begin{align}\label{eq:sum}
    \norm{\*z^{\top}\*D\*z} = \frac{1}{m}\abs*{\Sigma_{i=1}^{m}Y_i^2}
\end{align}
We use the exponential deviation inequality in Corollary 5.17 from \cite{vershynin2010introduction} to control the summation term in eq.~\ref{eq:sum} to give:
\begin{align}
    P\Big(\norm{\*z^{\top}\*D\*z} \geq \frac{\epsilon}{2}\Big) &= 
    P\Big(\frac{1}{m}\abs*{\Sigma_{i=1}^{m}Y_i^2} \geq \frac{\epsilon}{2}\Big) \\ \nonumber
    &\leq 2\exp\Bigg[-c\min\bigg(\frac{\epsilon^2}{4K_e^2},\frac{\epsilon}{2K_e}\bigg)m\Bigg] \nonumber
\end{align}
Note that $\epsilon:=max(\delta, \delta^2)$. If $\delta\geq1$ then $\epsilon=\delta^2$. Thus, $\min(\epsilon, \epsilon^2) = \delta^2$. Using this and the fact that $K\geq2\norm{Y_i}_{\psi_2}\geq1$, we get
\begin{align}
     P\Big(\norm{\*z^{\top}\*D\*z} \geq \frac{\epsilon}{2}) \leq 2\exp\Bigg[-\frac{c_1}{K^4}\delta^2m\Bigg] \leq 2\exp\Bigg[-\frac{c_1}{K^4}(C^2d + t^2)\Bigg]
\end{align}
by substituting $\delta = C\sqrt{\frac{d}{m}} + \frac{t}{\sqrt{m}}$ and using $(a+b)^2\geq a^2+b^2$.

\textbf{Step 3: Union Bound.} Using Boole's inequality to compute the union bound over all the vectors $\*z$ in the net $\mathcal{N}$ with cardinality $\abs*{\mathcal{N}}=9^d$, we get
\begin{align}
    P\Bigg\{\max_{\*z\in\mathcal{N}}\norm{\frac{1}{m}\*A^{\top}\*A - \*{\Sigma}} \geq \frac{\epsilon}{2}\Bigg\} \leq 9^d\cdot2\exp\Bigg[-\frac{c_1}{K^4}(C^2d + t^2)\Bigg]
\end{align}
Pick a sufficiently large $C=C_K\geq K^2\sqrt{\log9/c_1}$, then the probability
\begin{align}
    P\Bigg\{\max_{\*z\in\mathcal{N}}\norm{\frac{1}{m}\*A^{\top}\*A - \*{\Sigma}} \geq \frac{\epsilon}{2}\Bigg\} &\leq \frac{2}{\exp\Big(d+\frac{c_1t^2}{K^4}\Big)} \\ \nonumber 
    &\leq 2\exp{\Big(-\frac{c_1t^2}{K^4}\Big)}
\end{align}
Thus with a high probability of at least $1-2\exp{(-ct^2)}$ eq.~\ref{eq:view-matrix} holds. In other words, the deviation of the subsampled view matrix from the entire view space, in spectral sense is $\mathcal{O}(d/m)$
\end{proof}

\begin{lemma}
(Subsampled within-view covariance bound) Let $\*X$ be the $N\times m\times d$ tensor whose elements $\*A\in\mathbb{R}^{m\times d}$ are identically distributed matrices with rows $\*A_i$ representing $m$-views sampled from a larger set of views in $\mathbb{R}^d$. If $\*A_i$ are independent sub-gaussian vectors with second moment $\*{\Sigma}_w$, then for every $t\geq0$, with probability at least $1-2\exp{(-ct^2)}$, we have 
\begin{equation}
    \norm{{\*R}_w - \*{\Sigma_w}} \leq N\frac{C^2d+t^2}{m}
\end{equation}
Here ${\*R}_w$ is the sum of within-view covariance matrices for $m$ views and $C>0$ depends only on the sub-gaussian norm $K=\max_i\norm{\*A_i}_{\psi_{2}}$ of the subsampled view space.
\end{lemma}\label{lemma:rw}
\begin{proof}
Let us now consider the rearranged $m$-view subsampled data tensor $\*X\in \mathbb{R}^{N\times m\times d} = [\*A^{(1)}, ..., \*A^{(N)}]$. Let $\*A$ be the $m\times d$ view-specific data sampled identically for $N$ times. 
Without loss of generality, assume the rows to be zero mean which makes covariance computation simpler.
The rows $\*A_i$ are independent sub-gaussian vectors with second moment matrix $\*{\Sigma}=\mathbb{E}\*A^{\top}\*A$. The between-view covariance matrix ${\*R}_w$ for $m$ views can be written as:
\begin{align}
    {\*R}_w = \frac{1}{m}\sum_{i=1}^{N}\sum_{j=1}^{m}\*A_j\otimes\*A_j = \sum_{i=1}^{N}\frac{1}{m}\*A^{(i)\top}\*A^{(i)}
\end{align}
The matrix $\*A$ is a sampling of $m$ views from an unknown and larger number of views $M$ for which the $\*R_w$ is constructed. We want to bound the difference between this term and the within-view covariance of the whole space using lemma \ref{lemma:view-matrix}:
\begin{flalign*}
    \norm{{\*R}_w - \*{\Sigma_w}} &= \norm{\sum_{i=1}^{N}\frac{1}{m}\*A^{(i)\top}\*A^{(i)} - \sum_{i=1}^{N}\*{\Sigma}_w^{(i)}} \\
    &= \norm{\sum_{i=1}^{N}\frac{1}{m}\*A^{(i)\top}\*A^{(i)} - \*{\Sigma}_w^{(i)}} \\
    &\leq \sum_{i=1}^{N}\norm{\frac{1}{m}\*A^{(i)\top}\*A^{(i)} - \*{\Sigma}_w^{(i)}} &&\text{[Triangular inequality]} \\
    &= N\norm{\frac{1}{m}\*A^{\top}\*A - \mathbb{E}\*A^{\top}\*A} &&\text{[Identical sampling]}\\
    &\leq N\max{(\delta, \delta^2)} \quad\text{with}\quad \delta=C\sqrt{\frac{d}{m}} + \frac{t}{\sqrt{m}} &&\text{[From lemma \ref{lemma:view-matrix}]}\\
    &= N\cdot \Big(\sqrt{\frac{Cd+t}{m}} \Big)^2\\
    &\leq N\cdot \Big(\frac{C^2d+t^2}{m}\Big) &&\text{[$d,m>1$ and $(a+b)^2\leq a^2+b^2$]}
\end{flalign*}
\end{proof}

\newpage
\section{Lemma: Upper Bound for Bootstrapped Total-View Covariance}\label{sec-lemma:rt}
\begin{lemma}\label{lemma:rt}
(Subsampled total-view covariance bound) 
Let $\*X$ be the $N\times m\times d$ tensor whose elements $\*A\in\mathbb{R}^{m\times d}$ are identically distributed matrices with rows $\*A_i$ representing $m$-views sampled from a larger set of views in $\mathbb{R}^d$.
Construct a total-view matrix ${\*X}\in\mathbb{R}^{m\times d}$ by summing entries across all views. Let $\*{\Sigma}_t$ be the second moment of the total-view space. Then, we have 
\begin{equation}
    \norm{{\*R}_t - \*{\Sigma_t}} \leq Nm
\end{equation}
Here ${\*R}_t$ is the total-view covariance matrix and $c_2>0$ depends on the range of the total view space $k$ such that $\abs{{\*X}}\leq k$.
\end{lemma}
\begin{proof}
Consider the $m$-view subsampled data tensor rearranged with feature vectors as rows to get $\*X\in \mathbb{R}^{N\times m\times d} = [\*A^{(1)}, ..., \*A^{(N)}]$ with rows of $\*A^*$ as $\*A_i$.
Without loss of generality, assume the $d$-dimensional rows of $\*A$ to be zero mean which makes estimating covariances simpler.
The covariance ${\*R}_t$ of the total view matrix can be written as follows
\begin{align}
    {\*R}_t 
    = \frac{1}{m}\displaystyle\sum_{i=1}^{N}\Big(\sum_{i=1}^m\*A^{(i)}\Big)\Big(\sum_{i=1}^m\*A^{(i)}\Big)^{\top} = \frac{1}{m}\displaystyle\sum_{i=1}^{N}\Big(\sum_{j=1}^{m}\*A_j^{(i)}\Big)\otimes\Big(\sum_{j=1}^{m}\*A_j^{(i)}\Big)\\\nonumber
    = \frac{1}{m}\displaystyle\sum_{i=1}^{N}\Big(\sum_{j=1}^{m}\*A_j^{(i)}\Big)\otimes\Big(\sum_{j=1}^{m}\*A_j^{(i)}\Big)=\frac{1}{m}\displaystyle\sum_{i=1}^{N}\*W_i\*W_i^{\top}\nonumber
\end{align}
We want to bound the difference between this subsampled total-view covariance matrix and the second moment of the total-view space. 
Let ${\*a}^{(i)} = \sum_{j=1}^{m}\*A_j^{(i)}$ for $i=1,\dotsc,N$. The vector ${\*a}^{(i)}$ is the sum-of-views. We use a useful application of Jensen's inequality here: $\norm{X-\mathbb{E}X} \leq 2\norm{X}$ with $\abs*{\mathbb{E}X}\leq\mathbb{E}\abs*{X} \leq \abs*{X}$
\begin{flalign*}
\norm{{\*R}_t - \*{\Sigma}_t} 
    &= \frac{1}{m}\norm{\sum_{i=1}^{N}{\*a}^{(i)}{\*a}^{(i)\top} - \sum_{i=1}^{N}\*{\Sigma}_t^{(i)}} \\
    &\leq \frac{1}{m}\sum_{i=1}^{N}\norm{{\*a}^{(i)}{\*a}^{(i)\top} - \*{\Sigma}_t^{(i)}} &&\text{[Triangular inequality]} \\
    &= \frac{N}{m}\norm{{\*a}{\*a}^{\top} - \mathbb{E}{\*a}{\*a}^{\top}} &&\text{[Identical sampling]}\\
    &\leq \frac{N}{m}\norm{\*a\*a^{\T}} &&\text{[Triangular and Jensen's inequality]}\\
    &= \frac{N}{m}m^2 = Nm&&\text{[From assumption: $\norm{\*a}_2=1$]}
\end{flalign*}
\end{proof}

\newpage
\section{Theorem: Error of the Bootstrapped Multi-view Correlation}\label{thm:rho}

\begin{theorem}{
Let $\*X=[\*A^{(1)},\dotsc,\*A^{(N)}]$ be the $m\times d$ matrices of $m$ views sampled from an unknown number of views $M$. Let the rows $\*A_l$ of the view matrices $\*A$ be independent subgaussian vectors in $\R^d$ with $\norm{\*A_l}_2=1:l=1,\dotsc,m$. Then for  any $t\geq0$, with probability at least $1-2\exp({-ct^2})$, we have
\begin{flalign*}
{\rho_m}\leq \max\bigg(1, C\frac{m^2}{(\sqrt{d}+t)^2}\rho^{*}\bigg)
\end{flalign*}
\textit{Here, $\rho_m$ and $\rho^*$ are the mv-corr objectives for subsampled views $m$ and the total number of views $M$ respectively. The constant $C$ depends only on the subgaussian norm $K$ of the view space, with $K = \displaystyle\max_{i,l}\norm{\*A_l^{(i)}}_{\psi_2}$}
}\end{theorem}

\begin{proof}
Starting from the objective defined in the main paper and ignoring the normalization factors, the objective $\rho_m$ for $m$ views can be rewritten as:
\begin{align*}
    \rho_m = \frac{\text{Tr}(\*W^{\top}{\*R}_B\*W)}{\text{Tr}(\*W^{\top}{\*R}_W\*W)} 
    = \frac{\<{\*R_B} + \*{\Sigma}_B - \*{\Sigma}_B, \*W\*W^{\top}\>}{\<{\*R}_W + \*{\Sigma}_W - \*{\Sigma}_W, \*W\*W^{\top}\>}\\
    \leq \frac{\<{\*{\Sigma}}_B, \*W\*W^{\top}\> + \norm{{\*R}_T - \*{\Sigma}_T} + \norm{{\*R}_W - \*{\Sigma}_W}}{\<{\*{\Sigma}}_W, \*W\*W^{\top}\> - \norm{{\*R}_W - \*{\Sigma}_W}}
\end{align*}
where $\*{\Sigma}_b$ and $\*{\Sigma}_w$ are the second moment matrices for the the between-view and within-view covariances respectively. This can be written using cyclical properties of trace function and relation between spectral norm and trace.
Additionally note from the previous result that we can use total covariance to simplify the estimation of $\*R_B$. That is, $\*R_B = \*R_T-\*R_W$. The rest follows through triangular inequalities.

Observe that the ratio $\<{\*{\Sigma}}_B, \*W\*W^{\top}\>/\<{\*{\Sigma}}_W, \*W\*W^{\top}\>$ is the optimal $\rho^*$ we are interested to bound the approximation $\rho_m$ from. We can show that $\abs*{\rho}\leq 1$. Additionally the two trace terms are sum of normalized eigen values (each bounded above by 1). Thus $\<{\*{\Sigma}}_B, \*W\*W^{\top}\>\in[1,d]$ and $\<{\*{\Sigma}}_W, \*W\*W^{\top}\>\in[1,d]$. Furthermore, from lemma~\ref{lemma:rw}, we know that the norm term with $\*R_W$ is greater than 1 i.e., $\norm{\*R_T - \*\Sigma_W}\leq C\frac{d}{m} > 1$, because we always choose the embedding size to be greater than the number of views subsampled. With these inequalities. We can loosely bound the above inequality for $\rho_m$ as:

\begin{align*}
    \rho_m &\leq \frac{\<{\*{\Sigma}}_B, \*W\*W^{\top}\>}{\<{\*{\Sigma}}_W, \*W\*W^{\top}\>}\frac{ \norm{{\*R}_T - \*{\Sigma}_T} + \norm{{\*R}_W - \*{\Sigma}_W}}{ \norm{{\*R}_W - \*{\Sigma}_W}}\\
    &\leq \rho^{*}\frac{ \norm{{\*R}_T - \*{\Sigma}_T} }{ \norm{{\*R}_W - \*{\Sigma}_W}} \leq \rho^*\frac{2Nm}{NC\frac{(\sqrt{d}+t)^2}{m}} &&\text{[From Lemmas \ref{lemma:rw} and \ref{sec-lemma:rt}]}\\
    &\leq C'\rho^*\frac{m^2}{d} \approx \mathcal{O}(\frac{m^2}{d}) \\
\end{align*}
where $C'$ is a constant term that depends only the subgaussian norm of the $d$-dimensional feature vectors.
\end{proof}

\newpage
\bibliography{references}
\bibliographystyle{icml2020}
\end{document}